\documentclass[journal,10pt]{IEEEtran}
\usepackage{amsmath,amssymb,amsfonts,amsthm}
\usepackage{algorithm}
\usepackage[noend]{algorithmic}
\usepackage{graphicx}
\usepackage{times}
\usepackage[numbers]{natbib}
\usepackage[bookmarks=true]{hyperref}
\usepackage{xspace}
\usepackage{tabularx,threeparttable}
\usepackage{subcaption}

\newcommand{\mtm}{\emph{multi-to-multi}\xspace}
\newcommand{\mts}{\emph{multi-to-single}\xspace}
\newcommand{\sts}{\emph{multi-to-single-restricted}\xspace}
\newcommand{\dtd}{\emph{single-to-single}\xspace}

\newcommand{\cte}{\emph{full-to-edge}\xspace}
\newcommand{\ctc}{\emph{full-to-full}\xspace}
\newcommand{\ete}{\emph{edge-to-edge}\xspace}

\newcommand{\AND}{{\sc and}\xspace}
\newcommand{\OR}{{\sc or}\xspace}

\newcommand{\ignore}[1]{}

\def\C{\mathcal{C}}

\def\F{\mathcal{F}}

\def\W{\mathcal{W}}

\def\Naturals{\mathbb{N}}
\renewcommand{\leq}{\leqslant}
\renewcommand{\geq}{\geqslant}

\newcommand{\Cpp}{C\raise.08ex\hbox{\tt ++}\xspace}

\newcommand{\mc}{multi-configuration\xspace}
\newcommand{\mcs}{multi-configurations\xspace}

\newcommand{\pspace}{{\sc pspace}\xspace}
\newcommand{\np}{{\sc np}\xspace}

\newtheorem{defin}{Definition}
  \newenvironment{definition}{\begin{defin} \sl}{\end{defin}}
\newtheorem{theo}[defin]{Theorem}
  \newenvironment{thm}{\begin{theo} \sl}{\end{theo}}
\newtheorem{lemma}[defin]{Lemma}
  
\newtheorem{propo}[defin]{Proposition}
  
\newtheorem{coro}[defin]{Corollary}
  
\newtheorem{obse}[defin]{Observation}

  \renewenvironment{thebibliography}[1]{%
    \begin{oldthebibliography}{#1}%
      \footnotesize
      \setlength{\parskip}{0ex}%
      \setlength{\itemsep}{0ex}%
  }%
  {%
    \end{oldthebibliography}%
  }
\def\mypart#1#2{%
  \par\break % Page break
  \vskip .7\vsize % Vertical shift
  \refstepcounter{part}% Next part
  {\centering\Huge Part \thepart.\par \\
   \centering #1
  }%
  \vskip .1\vsize % Vertical shift

  #2
  \vfill\break % Fill the end of page and page break
}

\def\mypart#1#2
{
	\par\newpage\clearpage % Page break
	\vspace*{5cm} % Vertical shift
	\refstepcounter{part}% Next part
	{
		\centering
		\textbf
		{
			\Huge Part \thepart
		}
		\par
	}
	\vspace{1cm}% Vertical shift
	{
		\centering
		\textbf
		{
			\Huge #1
		}
		\par
	}
	\vspace{4cm}% Vertical shift
	#2
	\vfill
	\pagebreak % Fill the end of page and page break
}

\newboolean{ShowTODO}
\setboolean{ShowTODO}{true}

\ifthenelse{\boolean{ShowTODO}}{%

  \def\marrow{{\raggedright\footnotesize $\longleftarrow$}}

  \def\kiril#1{\textcolor{blue}{{\sc Kiril says: }{\marrow\sf #1}}}

  \def\aviv#1{\textcolor{red}{{\sc Aviv says: }{\marrow\sf #1}}}

  \def\markdb#1{\textcolor{red}{{\sc Mark says: }{\marrow\sf #1}}}

  \def\danny#1{\textcolor{red}{{\sc Danny says: }{\marrow\sf #1}}}

}{%

  \def\kiril#1{}

  \def\aviv#1{}

  \def\markdb#1{}

  \def\danny#1{}

}

\begin{document}

\title{On the hardness of unlabeled\\ multi-robot motion
  planning\thanks{\authorrefmark{1}This work has been supported in
    part by the Israel Science Foundation (grant no.\ 1102/11), by the
    German-Israeli Foundation (grant no.\ 1150-82.6/2011), and by the
    Hermann Minkowski-Minerva Center for Geometry at Tel Aviv
    University.}}

\author{\IEEEauthorblockN{Kiril Solovey\authorrefmark{1} and Dan Halperin\authorrefmark{1}}\\
  \IEEEauthorblockA{Blavatnic School of Computer Science\\
    Tel Aviv University, Israel\\
    email: \{kirilsol,danha\}@post.tau.ac.il}}

\maketitle

\begin{abstract}
  In \emph{unlabeled} multi-robot motion planning several
  interchangeable robots operate in a common workspace. The goal is to
  move the robots to a set of target positions such that each position
  will be occupied by \emph{some} robot. In this paper, we study this
  problem for the specific case of unit-square robots moving amidst
  polygonal obstacles and show that it is PSPACE-hard. We also
  consider three additional variants of this problem and show that
  they are all PSPACE-hard as well.  To the best of our knowledge,
  this is the first hardness proof for the unlabeled
  case. Furthermore, our proofs can be used to show that the
  \emph{labeled} variant (where each robot is assigned with a specific
  target position), again, for unit-square robots, is PSPACE-hard as
  well, which sets another precedence, as previous hardness results
  require the robots to be of different shapes.
\end{abstract}

\IEEEpeerreviewmaketitle

\section{Introduction} \label{sec:intro} In practical settings where
multiple robots operate in a common environment it is often the case
that the robots are identical in form and functionality and thus are
interchangeable. Specifically, in \emph{unlabeled} multi-robot motion
planning (unlabeled planning, in short) a group of identical robots
need to reach a set of target positions. As the robots are identical
we only require that in the end of the process each target position
will be occupied by \emph{some} robot. This is in contrast to the
standard \emph{labeled} (also known as \emph{fully-colored})
multi-robot motion problem, where each robot is required to reach a
\emph{specific} target position, and the robots may differ in
shape. While labeled planning has been of interest to many researchers
for the past four decades, unlabeled planning has only been recently
introduced and investigated.

\subsection{Related work}
We start with the much more intensively studied \emph{labeled} case of
multi-robot motion planning.  Schwartz and Sharir~\cite{ss-pm3} were
the first to consider the labeled problem from the geometric
point-of-view, and in particular studied the case of two discs moving
amidst polygonal obstacles and developed an algorithm with a running
time of $O(n^3)$, where $n$ is the complexity of the
workspace. Yap~\cite{c-cms84} also considered this setting and
described an algorithm of complexity $O(n^2)$. Later on, Sharir and
Sifrony~\cite{ss-cmp91} proposed an $O(n^2)$ algorithm as well,
although their algorithm deals with several additional types of
robots, besides discs.

When the number of robots is no longer a fixed constant the problem
can become computationally intractable. Specifically, Hopcroft et
al.~\cite{hss-cmpmio} showed that the problem is \pspace-hard for the
setting of rectangular robots bound to translate in a rectangular
workspace. Their proof required the rectangular robots to be of
varying dimensions.  Spirakis and Yap~\cite{sy-snp84} showed that the
problem is \np-hard for disc robots in a simple-polygon workspace; here the proof strongly relies on the fact that the discs are of
varying radii.

More recently, Hearn and Demaine~\cite{hd-psb05,hd-gpc09} improved the
result of Hopcroft et al.\ by showing that the robots can be
restricted to only two types---$2\times1$ and $1\times 2$
rectangles. Their work is more general: They introduced in this work
the \emph{nondeterministic constraint logic} (NCL) model of
computation, for which they describe several \pspace-hard problems,
and from which they derive the \pspace-hardness of a variety of
puzzle-like problems that consist of sliding game pieces. (We describe
the NCL model in detail later on.) In particular, they applied their
technique to the \emph{SOKOBAN} puzzle, where multiple ``crates'' need
to be pushed to target locations, and the \emph{Rush Hour} game, where
a parking attendant has to evacuate a car from a parking lot, by
clearing a route blocked by other cars. Other hardness results using
the NCL model have
followed~\cite{bb-rb12,bg-dyn13,hj-croll12,ikd-rec12}.

Due to the various hardness results related to the problem, the
interest of the Computational-Geometry community in the multi-robot
motion planning problem has diminished over the years and the
attention has gradually shifted to the Robotics community who started
to develop sampling-based tools for the problem. Sampling-based
algorithms\footnote{We only briefly mention this area of research, as it is
  beyond the scope of this paper.} try to capture the structure of the
configuration space of the problem (or more accurately, its division
into ``free'' and ``forbidden'' regions) by random sampling of the
space and connecting nearby configurations by ``simple'' paths, to
form a \emph{roadmap}. While such methods are incapable of determining
whether a solution does not exist, they often provide asymptotic
guarantees of completeness and optimality. Even though such techniques
can be applied as-is to the labeled planning problem~\cite{sl-upp},
many approaches that were specifically designed for the problem have
emerged~\cite{hh-hmp,shh12, ssh-fne13,so-cppmr,wc-cmpp11}.

Unlabeled multi-robot motion planning was introduced by Kloder and
Hutchinson~\cite{kh-pppi05}, who described a sampling-based algorithm
for the problem. More recently, Solovey and Halperin have developed a
sampling-based algorithm for the unlabeled problem, as well as for a
generalization termed $k$-\emph{color} planning that consists of $k$
groups of interchangeable robots~\cite{sh-kcolor14}. Krontiris et
al.~\cite{lsdlyb-similar_part_rearrange14} describe an adaptation of
this work for the problem of rearranging several objects using a
robotic manipulator. Turpin et al.~\cite{tmk-cap13} describe an
efficient and complete algorithm for unlabeled planning for disc
robots which also guarantees finding the optimal solution in terms of
the length of the longest path traversed by a robot. However, their
algorithm makes the assumption that a certain portion of the free
space, surrounding each start or target position, is
\emph{star-shaped}. More recently, Adler et
al.~\cite{abhs-unlabeled14} studied a similar setting with unlabeled
disc robots operating in a simple polygon, although under a milder
assumption requiring each pair of start or target positions to be
separated by a distance of at least $4$, where the radius of the
robots is $1$. They describe an algorithm with a running time of
$O(n\log n+ mn+m^2)$, where $n$ is the complexity of the polygon and
$m$ is the number of robots. A crucial questions that follows from
their work is whether the efficient solution of the problem is
possible due to the separation constraints or the fact that the robots
are unlabeled.

\subsection{Contribution}
In this paper we study the problem of unlabeled multi-robot motion
planning for unit-square robots moving amidst polygonal obstacles.  We
show that the decision problem, namely, to decide whether a solution
exists, is \pspace-hard. To the best of our knowledge, this is the
first hardness proof for the unlabeled case. In fact we consider four
variants of the unlabeled problem (see Section~\ref{sec:prel} for a
precise definition) and show that they are all \pspace-hard. For
instance, we show that the seemingly easier version of the problem
where only one of the robots is required to reach a certain target
position while the other robots function as movable obstacles, is
also computationally intractable. We mention that our proofs can be
used to show that the labeled variant, again, for unit-square robots,
is \pspace-hard as well, which sets another precedence, as previous
hardness results require the robots to be of different
shapes~\cite{hd-psb05,hss-cmpmio,sy-snp84}. The various hardness
results for multi-robot motion planning are summarized in
Table~\ref{tbl:contribution}. While our result in itself is negative,
we hope that it will motivate research of other variants of the
unlabeled problem which may turn out to be polynomially solvable.

This paper complements another work by the
authors~\cite{SolYuZamHal15} in which we study a slightly different
setting of the unlabeled problem and present an efficient algorithm to
tackle it. In particular we consider the problem of unlabeled motion
planning of unit-disc robots moving amid polygonal obstacles.  We show
that if two simplifying assumptions are made regarding the distances
between pairs of start and target positions and between such positions
and the obstacles, an efficient algorithm can be developed.  In
particular, our algorithm has a running time\footnote{For simplicity
  of presentation, $\log$ factors in the complexity of the algorithm
  are omitted, and hence the $\tilde{O}$ notation is used.}
$\tilde{O}(m^4+m^2n^2)$, where $m$ is the number of robots and $n$ is
the complexity of the obstacles. Furthermore, the algorithm returns a
solution whose total length (namely the total length traveled by all
the robots) is $\text{OPT}+4m$, where OPT is the optimal solution
cost. In spite of the difference in robots, we believe that the
hardness result presented in the current paper hints that such
mitigating separation (or other) assumption are essential in order to
obtain efficient algorithms as in the other paper.

The organization of this paper is as follows. In
Section~\ref{sec:prel} we describe the four variants of the unlabeled
problem that will be considered in this paper.  In
Section~\ref{sec:ncl} the NCL model of computation~\cite{hd-psb05},
which is a key ingredient in our hardness proof, is described. In
Section~\ref{sec:proof} we provide the hardness proofs.
\begin{table*}
  \footnotesize
  \begin{center}
    \begin{tabular}{|c||c||c|c|c|}
      \hline 
      Contributor & Problem & Complexity & Robots & Workspace \\ 
      \hline 
      \hline
      Hopcroft et al.~\cite{hss-cmpmio} & labeled & \pspace-hard & rectangles
                                                  & rectangle \\ 
      \hline 
      Spirakis, Yap~\cite{sy-snp84} & labeled & strongly \np-hard & discs
                                                  & simple polygon \\ 
      \hline 
      Hearn, Demaine~\cite{hd-psb05} & labeled & \pspace-complete & $1\times
                                                                    2$
                                                                    rectangles & rectangle \\ 
      \hline 
      \textbf{this paper} & unlabeled, labeled &                     \pspace-hard  & unit squares & polygonal obstacles \\ 
      \hline 
    \end{tabular} 
  \end{center}
  \caption{Hardness results related to the multi-robot problem.}
  \label{tbl:contribution}
\end{table*}

\section{Preliminaries}\label{sec:prel}
Let $r$ be a robot operating in the planar workspace $\W$. We denote
by $\C(r)$ the \emph{configuration space} of $r$, and by
$\F(r)\subset\C(r)$ the \emph{free space} of $r$---the collection of
all configurations for which the robot does not collide with
obstacles. Given $s,t\in \F(r)$, a \emph{path} for $r$ from $s$ to $t$
is a continuous function $\pi:[0,1]\rightarrow \F(r)$, such that
$\pi(0)=s, \pi(1)=t$.

We say that two robots $r,r'$ are \emph{geometrically identical} if
$\F(r) = \F(r')$ for the same workspace $\W$. Let
$R=\{r_1,\ldots,r_m\}$ be a set of $m$ geometrically identical robots,
operating in a workspace $\W$.  We use $\F$ to denote
$\mathcal{F}(r_i)$ for any $1\leq i \leq m$.
    
\begin{definition}
  A collection $C=\{c_1,\ldots,c_m\}$ of $m$ configurations is termed
  a \emph{multi-configuration}. Such a \mc is \emph{free} if
  $C\subset\F$ and for every two distinct configurations $c,c'\in C$
  it holds that $r(c)\cap r'(c')=\emptyset$, for $r,r'\in R$, where
  $r(x)$, for $x\in \C$ denotes the area covered by some robot
  $r\in R$ when placed in $x$.
\end{definition}

\begin{definition}\label{def:equivalent}
  Given two free \mcs, $C=\{c_1,\ldots,c_m\},C'=\{c'_1,\ldots,c'_m\}$,
  we say that they are equivalent ($C\equiv C'$) if there exist $m$
  paths $\Pi=\{\pi_1,\ldots,\pi_m\}$ that move the $m$ robots from $C$
  to $C'$. More formally, we demand that for every $c\in C$,
  there exists $1\leq i\leq m$ for which $\pi_i(0)=c$; for every
  $c'\in C'$ there exists $1\leq j\leq m$ for which $\pi_j(1)=c'$; for
  every $\tau\in [0,1]$ the \mc
  $\Pi(\tau)=\{\pi_1(\tau),\ldots,\pi_m(\tau)\}$ is free.
\end{definition} 
    
Given two equivalent \mcs $S,T$ we denote by
$\Pi(S,T)=\{\pi_1,\ldots,\pi_m\}$ a collection of $m$ paths that move
the robots from $S$ to $T$ by following the restrictions above.  Note
that Definition~\ref{def:equivalent} requires that every target
position will be occupied by \emph{some} robot, in contrast with the
classic definition which indicates which robot should reach where.
    
We define four decision problems that will each be shown to be
\pspace-hard below:
\begin{enumerate}
\item Given two free \mcs $S,T$, is it true that $S\equiv T$?
\item Given a free \mc $S$, and a configuration $t\in \F$, is there a
  \mc $T$ such that $t\in T$ and $S\equiv T$?
\item Given a free \mc $S$, a configuration $s\in S$, and another
  configuration $t\in \F$, is there a \mc $T$ such that $t\in T$,
  $S\equiv T$ for which there exists $\pi \in \Pi(S,T)$ such that
  $\pi(0)=s,\pi(1)=t$?
\item Given two configurations $s,t\in \F$, are there two \mcs $S,T$
  such that $s\in S,t\in T$ and $S\equiv T$?
\end{enumerate}
We will refer to these problems from now on as the \mtm, \mts, \sts,
and \dtd problems, respectively.  Note that \mts differs from \sts by
allowing any robot to reach~$t$. Although the four problems seem to be
closely related, it is not clear whether it is possible to reduce one
problem to another.

\section{Nondeterministic constraint logic}\label{sec:ncl}
In this section we review the \emph{nondeterministic constraint logic}
(NCL) model of computation, and state three decision problems that are
based on this model and are shown to be
\pspace-complete~\cite{hd-psb05}.  An NCL machine is defined by a
\emph{constraint} graph $G=(V,E)$, a \emph{weight} function
$w:E\rightarrow\Naturals$, and a \emph{minimum-flow} constraint
$c:V\rightarrow \Naturals$.

\begin{definition}
  A machine configuration is an orientation $o:E\rightarrow \{0,1\}$
  such that for every edge $(v,v')\in E$ it holds that
  $o(v,v')=1,o(v',v)=0$, or $o(v,v')=0,o(v',v)=1$. An orientation $o$
  is valid if for every $v\in V$ the sum of weights of the edges that
  are directed into $v$ is at least the minimum-flow constraint of the
  vertex. More formally,
  $$\forall v\in V: \sum_{v'\in N(v)}o(v',v)\cdot w(v',v)\geq c(v),$$
  where $N(v)$ denotes the set of neighbors of $v$ in $G$.
\end{definition}

A \textit{move} from one orientation to another consists of a reversal
of the orientation of a single edge, while maintaining the
minimum-flow constrains. Given two orientations $o,o'$ we say that
they are \emph{equivalent}, and denote this relation by $o\equiv o'$,
if $o$ can be transformed into $o'$ by a series of moves. Given these
definition, the following decision problems are defined
in~\cite{hd-psb05}:

\begin{enumerate}
\item Given two orientations $o_S,o_T$, is it true that
  $o_S\equiv o_T$?
\item Given an orientation $o_S$ and an edge $(v,v')\in E$ is there
  another orientation $o_T$ such that $o_S\equiv o_T$ and
  $o_S(\{v,v'\})\neq o_T(\{v,v'\})$, i.e., the orientation of $(v,v')$
  is reversed between the two configurations?
\item Let $(v,v'),(u,u')\in E$. Additionally, let
  $o_{(v,v')},o_{(u,u')}$ be two orientations for these specific
  edges. Are there two configurations $o_S,o_T$ such that
  $o_S\equiv o_T$ and $o_S(v,v')=o_{(v,v')}, o_T(u,u')=o_{(u,u')}$?
\end{enumerate}

These problems are termed \ctc, \cte, and \ete, respectively. We are
interested in a particular setting of the problem where the constraint
graph is planar and consists of only two types of vertices:
\begin{itemize}
\item An \AND vertex has a minimum-flow constraint of two, and has
  exactly three incident edges, where one of the edges has a weight of
  two, and each of the other two edges has a weight of one.
\item An \OR vertex has a minimum-flow constraint of two and has
  exactly three incident edges, each with a weight of two.
\end{itemize}  

\begin{figure*}
  \centering
  \begin{subfigure}[b]{0.25\textwidth}
    \includegraphics[width=\textwidth]{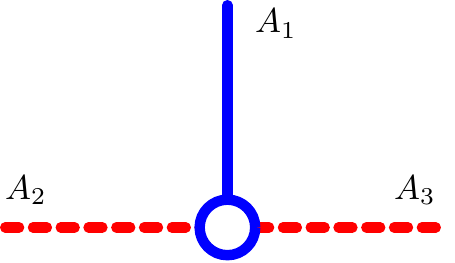}
    \caption{\AND vertex}
    \label{fig:graph_and}
  \end{subfigure}%
  ~ %add desired spacing between images, e. g. ~, \quad, \qquad, \hfill etc.
  \quad
  % (or a blank line to force the subfigure onto a new line)
  \begin{subfigure}[b]{0.25\textwidth}
    \includegraphics[width=\textwidth]{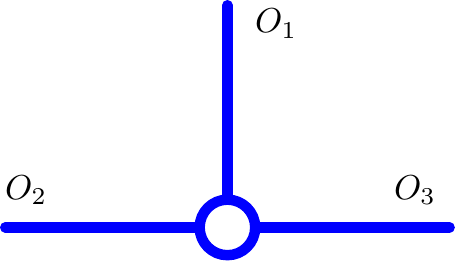}
    \caption{\OR vertex }
    \label{fig:graph_or}
  \end{subfigure}
  \caption{\AND and \OR vertices. Red (dashed) edges have a weight of
    1 and blue (solid) edges have a weight of $2$. The blue (solid)
    vertex (circle) represent an minimum flow constraint of $2$. In
    (a) $A_1$ can be directed outward if and only if $A_2$ and $A_3$
    are both directed inward. In (b) $O_1$ can be directed outward if
    and only if any of the other two is directed
    inward. }\label{fig:graph_vertices}
\end{figure*}

The following Theorem will play a central role in our analysis. Its
proof is found in~\cite[Theorem~12]{hd-psb05}.
\begin{theo}[Hearn and Demaine]\label{thm:ncl_hard}
  \emph{Full-to-edge}, \ctc, and \ete, are \pspace-complete,
  even when the constraints graph is simple, planar, and consists of
  only \AND and \OR vertices.
\end{theo}

\subsection{Grid-embedded constraint graph}
In order to simplify the reduction process in the following sections,
we show that given a constraint graph $G$, as described above, it can
be transformed into a two-dimensional constraint grid graph~$H$, such
that the three NCL decision problems remain \pspace-hard on~$H$ as
well. We mention that the authors of~\cite{hd-psb05} use a similar
grid embedding, but omit the relevant details. Thus we chose to
provide a full description of this process here.

We generate a new constraint graph $H$ whose vertices are grid
vertices and edges are grid edges. Each edge of $G$ is transformed
into a noncrossing path in $H$. Such an embedding is possible due to
the fact that $G$ is simple and planar. For the purpose of the
reduction it suffices to know that such an embedding can be carried
out in polynomial time, but we mention that a linear-time algorithm by
Liu et al.\ exists~\cite{lms-grid98}.

As $G$ is planar and has a degree of three (it is exclusively made of
\AND and \OR vertices), it can be embedded on a planar grid
$H=(V_H,E_H)$ which is defined as follows. The set of vertices of $H$
is defined to be $V_H:=V \cup U$, where $U$ is an additional set of
vertices called \textit{connectors} and where each $v\in V_H$
corresponds to a point on the grid. Every edge $(v,v')\in E_H$
corresponds to an axis-parallel segment that connects the two points
$v,v'$ on the grid. Given two vertices $v,v'\in V$ for which
$(v,v')\in E$ we denote by $H(v,v')=(v,u_1,\ldots, u_{\ell},v')$ the
path in $H$ that corresponds to $(v,v')$.

We also define the weight and capacity functions $w_H,c_H$,
respectively. Each vertex $v\in V$ maintains its original capacity of
two from $G$, that is, $c_H(v):=c(v)$. Let $(v,v')\in E$ and let
$u\in U$ such that $u\in H(v,v')$. Then
$c_H(u):=w(v,v')$. Additionally, suppose that $(u,u')$ represents an
edge on the path $H(v,v')$, then we let $w_H(u,u'):=w(v,v')$.

\begin{figure*}
  \centering
  \begin{subfigure}[b]{0.3\textwidth}
    \includegraphics[width=\textwidth]{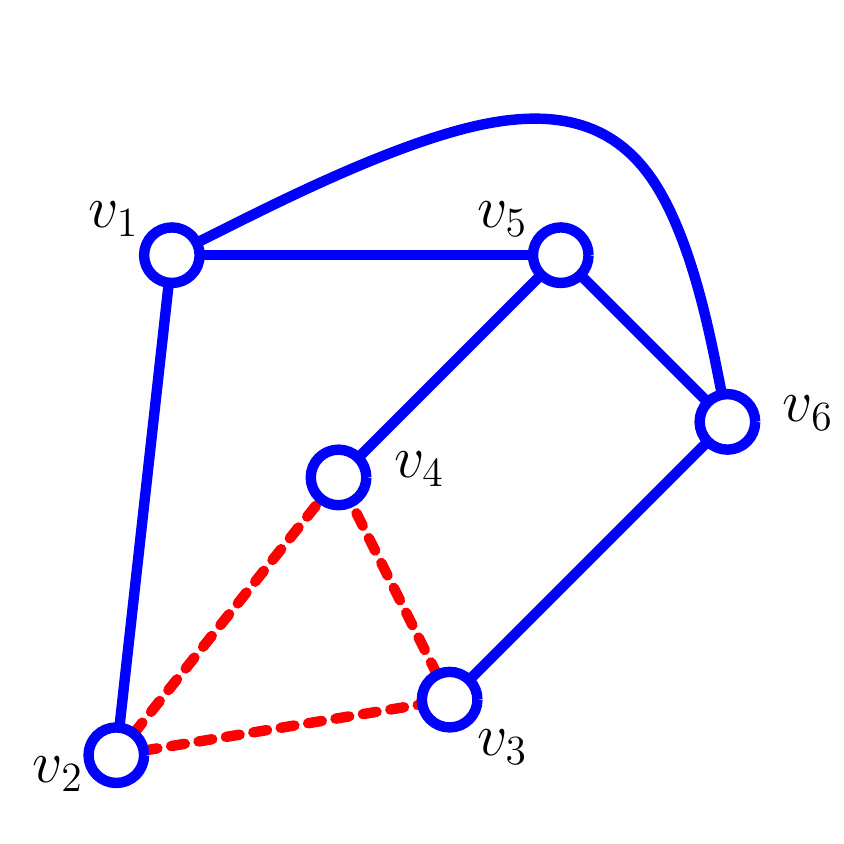}
    \caption{ }
    \label{fig:embedding0}
  \end{subfigure}%
  ~ %add desired spacing between images, e. g. ~, \quad, \qquad, \hfill etc.
  \quad
  % (or a blank line to force the subfigure onto a new line)
  \begin{subfigure}[b]{0.3\textwidth}
    \includegraphics[width=\textwidth]{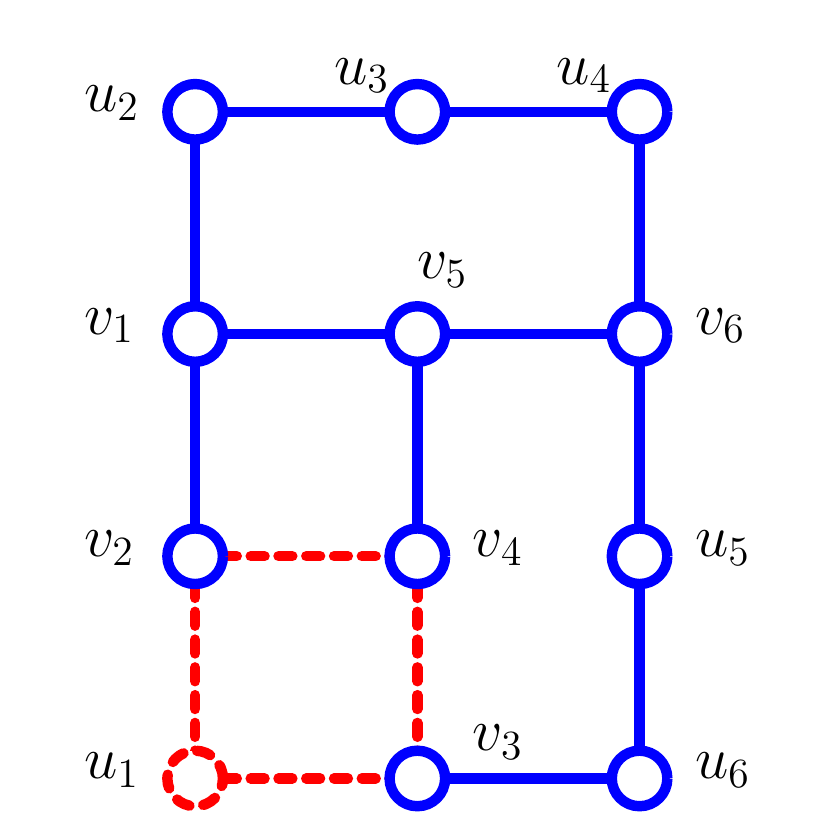}
    \caption{ }
    \label{fig:embedding1}
  \end{subfigure}
  \caption{A grid embedding of a simple planar graph with \AND and \OR
    vertices. In (a) we have the original graph $G=(V,E)$ while in (b)
    we have a grid embedding of it.  Recall that blue (solid) edges
    represent edges with weights of $2$, while red (dashed) represent
    weights of $1$. Similarly, blue circles represent vertices with
    min-flow constraints of $2$, while red represent vertices with
    min-flow constaints of $1$. Note that the embedding preserves the
    type of the vertices from $G$, e.g., $v_4$ is an \AND vertex in
    both cases. In order to make the embedding possible the
    \emph{connector} vertices $U=\{u_1,\ldots, u_6\}$ were added to
    $H$. Note that the minimum-flow costraint of $u_1$ is $1$. The
    edges of $G$ are represented as paths in $H$. For instance,
    $(v_1,v_6)\in E$ is represented by the path
    $H(v_1,v_6)=\{v_1,u_2,u_3,u_4,v_6\}$ in
    $H$. }\label{fig:grid_embedding}
\end{figure*}

\begin{lemma}
  \emph{Configuration-to-edge}, \ctc, and \ete, are \pspace-complete,
  even for the grid-embedded constraint graph $H$ that consists of
  only \AND,\OR, and connector vertices.
\end{lemma}

\begin{proof}
  Every orientation of $H$ can be transformed into an orientation for
  $H$, and vice versa.  Using this fact the hardness of the these
  three problems immediately follow from Theorem~\ref{thm:ncl_hard}.
\end{proof}

\section{From NCL to multi-robot motion planning}\label{sec:proof}
In this section we present the reduction from the three NCL problems,
which were described in the previous section, to our four unlabeled
multi-robot motion planning problems; we will call them unlabeled
problems for short. Specifically, we consider the case where the input
consists of a grid-embedded constraint graph~$H$, as described in
Section~\ref{sec:ncl}. Given such a graph $H$ we construct an
unlabeled scenario which corresponds to the graph and consists of
unit-square robots and polygonal obstacles. We use a grid layout, as
depicted in Figure~\ref{fig:cells}, where the cells of this grid are
of dimension $5\times 5$ and the walls separating the cells are of
thickness $1/2$. Each such cell functions as a placeholder for a
gadget which represents and emulates a specific vertex of $H$. The
gadgets are placed according to the positions of their counterpart
vertices in $H$. Note that between every two adjacent cells there is a
doorway of width $1$ so that an interaction between adjacent gadgets
can take place. We may rotate the gadgets depicted below by $90, 180$
or $270$ degrees so that gadgets that correspond to two vertices of
$H$ that share an edge will share a passage. When two vertices of $H$
share an edge, the corresponding gadgets will share a robot.  A
similar scheme was employed in~\cite{hd-psb05}, although they used
different gadgets as they were interested in showing the hardness of
slightly different motion planning problems.

\begin{figure}
  \centering
  \includegraphics[width=0.3\textwidth]{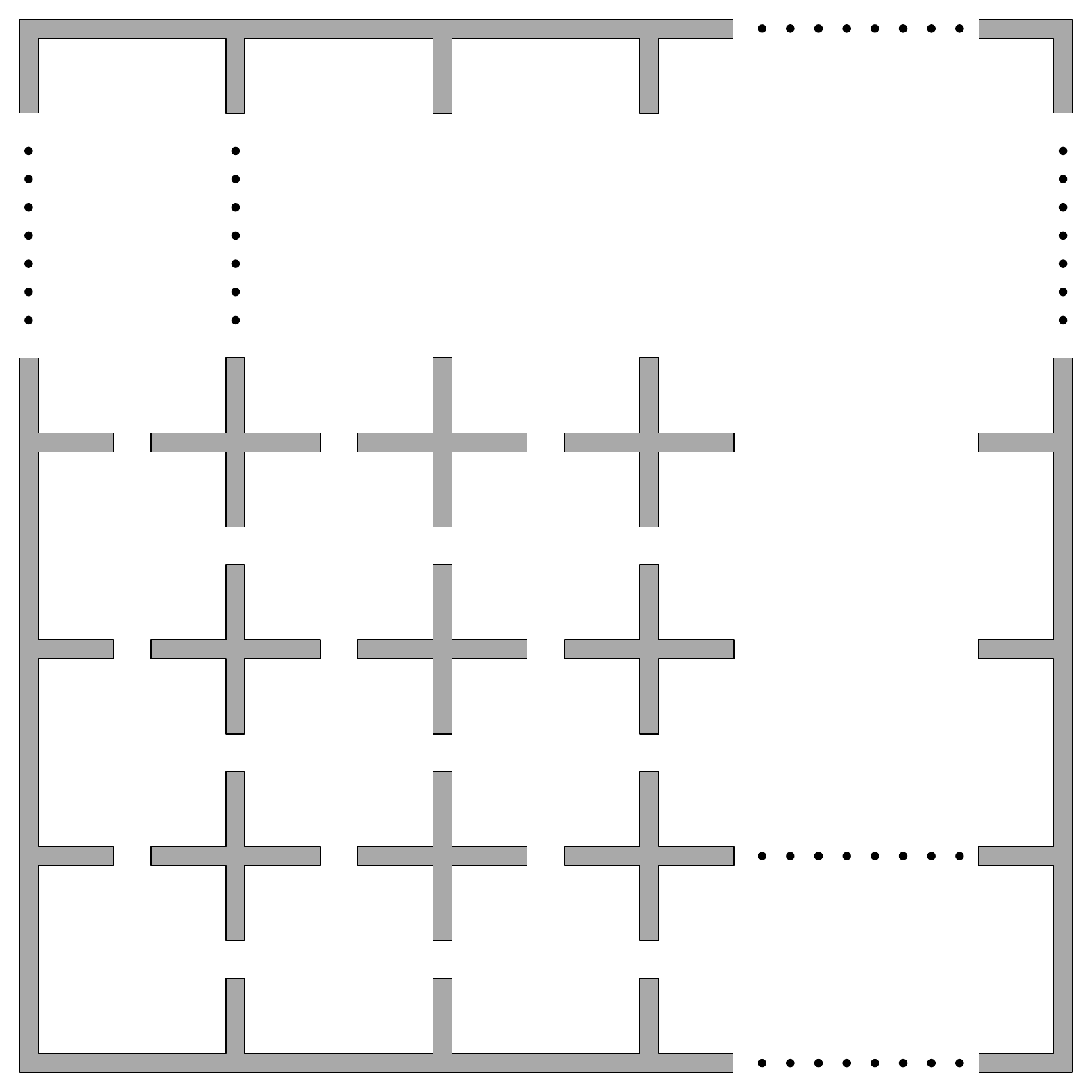}
  \caption{Placeholders for the gadgets.}\label{fig:cells}
\end{figure}

\subsection{\AND, \OR, and connector gadgets}\label{sec:gadgets}
For each vertex of $H$ we create a gadget that emulates the
functionality of this vertex in the NCL machine. For the vertices of
$U$ we create a \emph{connector} gadget, while for the vertices of $V$
we create \AND and \OR gadgets. All the gadgets fit into $5\times 5$
squares (see Figures~\ref{fig:con},\ref{fig:and},\ref{fig:or}) and
have either two or three exists through which they connect to other
gadgets. Every gadget accommodates several robots and contains
polygonal obstacles; the robots are illustrated in purple or green and
the obstacles are illustrated in gray. The white regions represent
portions of the free workspace. All the robots are placed such that
they neither overlap with the obstacles nor with each other. The \AND
gadgets also have a point obstacle, illustrated in red (its purpose
will be explained below). We mention that robots are allowed to touch,
but not penetrate, the obstacles. For an illustration of a connection
between two gadgets, see Figure~\ref{fig:edge_restricted}.

Every gadget accommodates several unit-square robots which fall into
two categories: those that never leave the gadget and those that may
penetrate the gadget or leave it to a neighbor gadget. The former are
called \emph{vertex} robots (drawn in purple), while the latter are
\emph{edge} robots (drawn in green). Edge robots are located at the
exits of the gadgets, one robot per exit. As the name suggests, edge
robots correspond to the edges of $H$ incident to the vertex. The
direction of the edge corresponds to the position of the robot, with
respect to the gadget. Specifically, an edge that is directed
\emph{inward} corresponds to an edge robot that is located at the exit
but does not penetrate the $5\times 5$ square of the gadget (for
example, robot $A_1$ in Figure~\ref{fig:and1}); an \emph{outward}
directed edge corresponds to an edge robot located at the exist such
that exactly half of it is located inside the $5\times 5$ square of
the gadget (see for example robots $A_2,A_3$ in
Figure~\ref{fig:and1}). We will refer to these two positions of the
edge robot as \emph{inside} a gadget, and \emph{outside} a gadget,
respectively. We note that when an edge robot of one gadget is located
outside, it is also inside the adjacent gadget, and vice versa. The
inverse relation between the position of the edge robots and the
orientation of the edges stems from the fact that we wish to avoid
situations where too few edges are directed into a vertex (and thus
the minimum-flow constraint is not satisfied), and situations where
too many edge robots are inside a gadget (and thus a collision
occurs). For example, in the \OR gadget in Figure~\ref{fig:or} it is
not possible for all the three edge robots to be simultaneously inside
the gadget, and this ensures that the corresponding \OR vertex in the
constraint graph will receive a sufficient amount of in-flow.

For simplicity, we only consider configurations of the robots where
the center of each robot is located at the vertices of a
$1/2\times 1/2$ grid. We refer to such configurations as
\emph{terminal}. For instance, all the robots in
Figures~\ref{fig:and},\ref{fig:con},\ref{fig:or} are placed at
terminal configurations. Additionally, the actual terminal
configurations are illustrated in Figure~\ref{fig:con2}. We also allow a robot to move
between two adjacent terminal configuration. The following three
lemmas illustrate the relation between the gadgets and the
vertices of $H$. Their proofs are straightforward and therefore
omitted.

\begin{figure*}
  \centering
  \begin{subfigure}[b]{0.35\textwidth}
    \includegraphics[width=\textwidth]{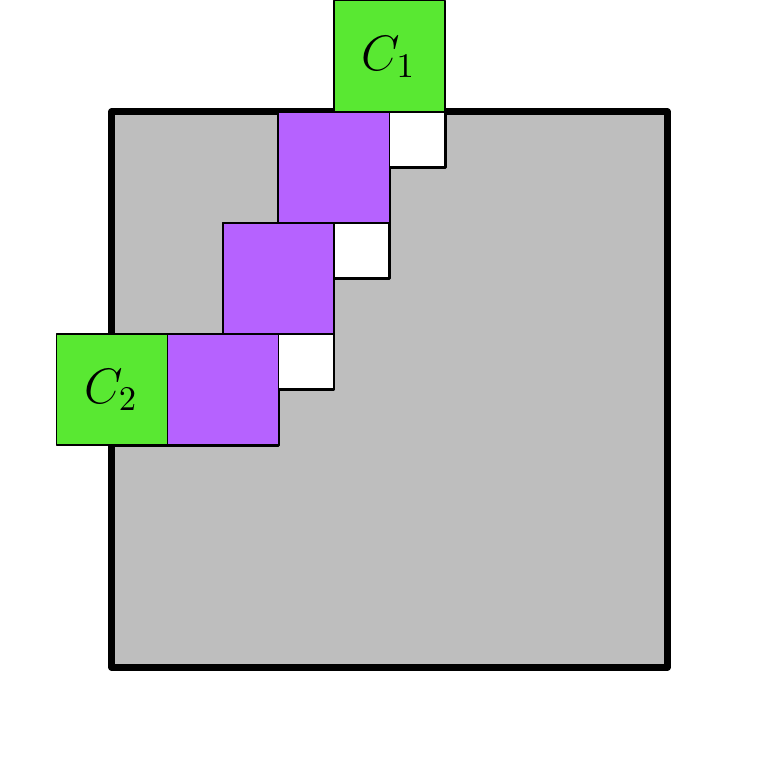}
    \caption{ }
    \label{fig:con1}
  \end{subfigure}%
  ~ %add desired spacing between images, e. g. ~, \quad, \qquad, \hfill etc.
          %(or a blank line to force the subfigure onto a new line)
  \begin{subfigure}[b]{0.35\textwidth}
    \includegraphics[width=\textwidth]{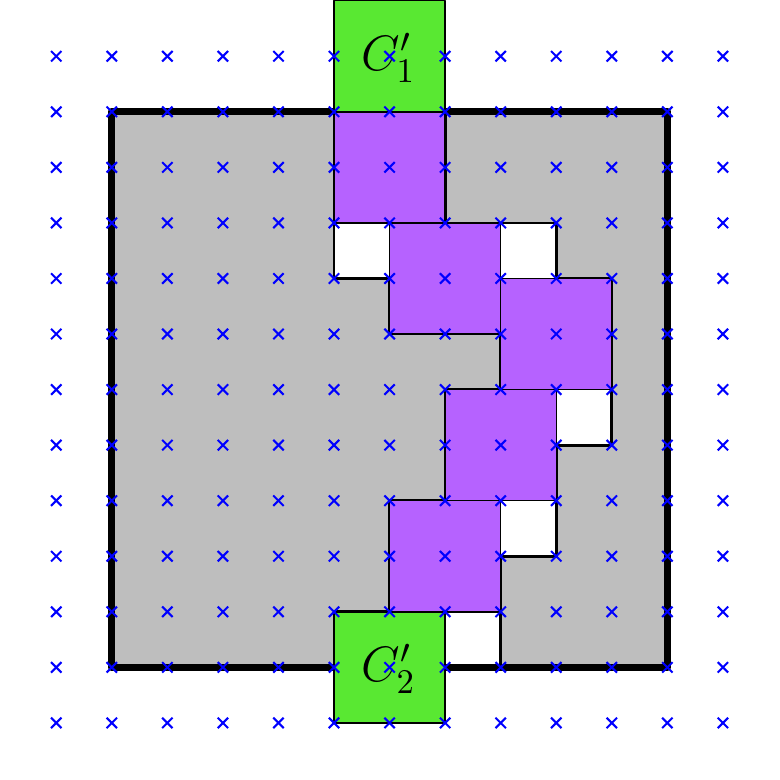}
    \caption{ }
    \label{fig:con2}
  \end{subfigure}
  \caption{Connector gadgets. The blue crosses in (b)
    represent terminal configurations, which are defined in
    Section~\ref{sec:gadgets}.}\label{fig:con}
\end{figure*}

\begin{figure*}
  \centering
  \begin{subfigure}[b]{0.35\textwidth}
    \includegraphics[width=\textwidth]{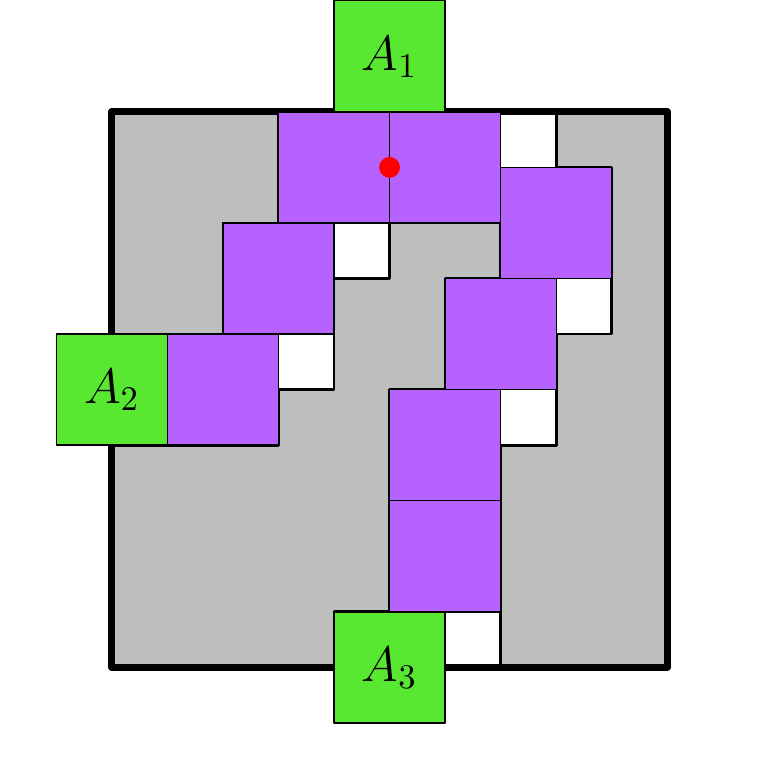}
    \caption{ }
    \label{fig:and1}
  \end{subfigure}%
  ~ %add desired spacing between images, e. g. ~, \quad, \qquad, \hfill etc.
  % (or a blank line to force the subfigure onto a new line)
  \begin{subfigure}[b]{0.35\textwidth}
    \includegraphics[width=\textwidth]{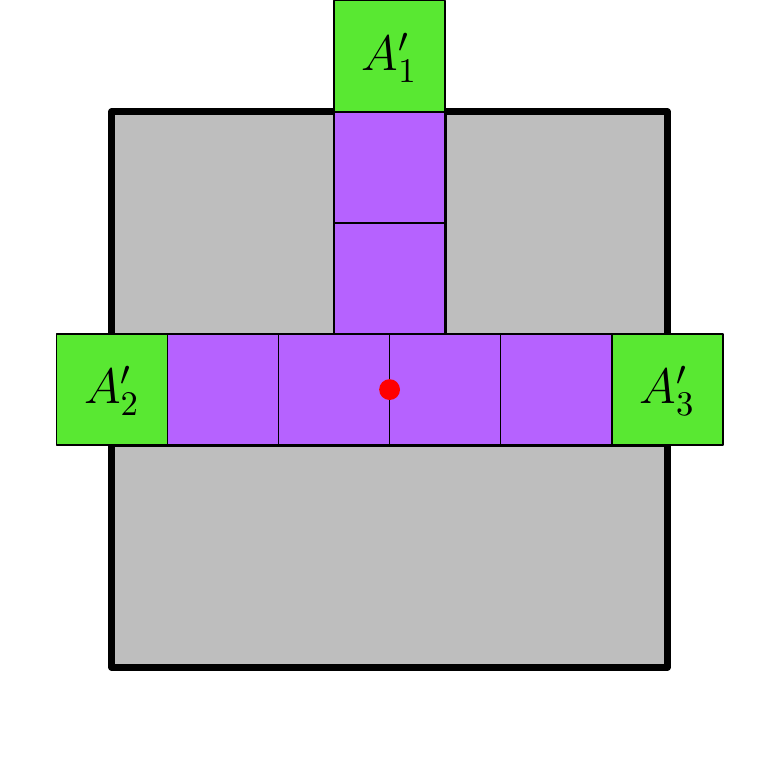}
    \caption{ }
    \label{fig:and2}
  \end{subfigure}
  \caption{\AND gadgets.}\label{fig:and}
\end{figure*}

\begin{figure*}
  \centering
  \begin{subfigure}[b]{0.35\textwidth}
    \includegraphics[width=\textwidth]{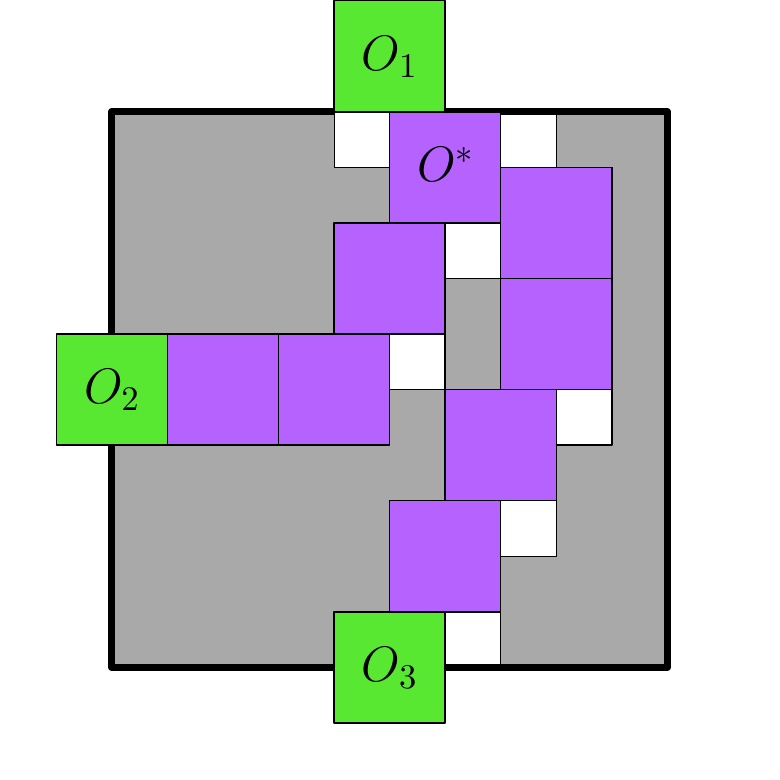}
  \end{subfigure}
  \caption{\OR gadget.}\label{fig:or}
\end{figure*}

\begin{lemma}\label{lem:con}
  Connector gadgets correspond to connector vertices in $H$, i.e., one
  of the two edge robot can be inside, if and only if the other edge
  robot is outside.
\end{lemma}
\begin{proof}
  See Figure~\ref{fig:con1} and Figure~\ref{fig:con2}.
\end{proof}
\begin{lemma}\label{lem:and}
  \AND gadgets correspond to \AND vertices in $H$, with
  $A_2,A_3,A'_2,A'_3$ corresponding to the $1$-weight edges, and
  $A_1,A'_1$ corresponding to $2$-weight edges, e.g., $A_1$ can move
  inside the gadget if and only if $A_2,A_3$ are both outside.
\end{lemma}
\begin{proof}
  See Figure~\ref{fig:and1} and Figure~\ref{fig:and2}.
\end{proof}
\begin{lemma}\label{lem:or}
  The \OR gadgets correspond to \OR vertices in $H$, i.e., one of the
  edge robots can move inside if and only if at least one of the other
  edge robots is outside.
\end{lemma}
\begin{proof}
	See Figure~\ref{fig:or}.
\end{proof}

\subsection{Unlabeled motion planning with gadgets}
We finalize the details of our reduction and prove its correctness. We
first show that the structure of gadgets is very restrictive and
allows a limited set of movements for the robots, so that robots
cannot wander between different gadgets.

\begin{lemma}\label{lem:restricted_edge}
  Each edge robot can be in at most two distinct terminal
  configurations.
\end{lemma}

\begin{proof}
  We show that for every possible connection of two gadgets the edge
  robot can be in at most two terminal configurations. We only
  consider here the combination of the second \AND gadget
  (Figure~\ref{fig:and2}) and the \OR gadget
  (Figure~\ref{fig:or}). The other combinations can be treated
  analogously. Specifically, we consider the case where the connection
  is made through the edge robot $A'_3=O_2$ (colored in orange), as
  illustrated in Figure~\ref{fig:edge_restricted}. Notice that robot
  $A'_3=O_2$ is stuck between the robots $D$ and $E$: $D$ is blocked
  to the left by the red point obstacle (Figure~\ref{fig:edge1}); $E$
  is blocked to the right by an obstacle
  (Figure~\ref{fig:edge2}). From this observation we conclude that
  every edge robot can either be inside or outside, and not in any
  other terminal configuration.
\end{proof}

\begin{figure*}
  \centering
  \begin{subfigure}[b]{0.45\textwidth}
    \includegraphics[width=\textwidth]{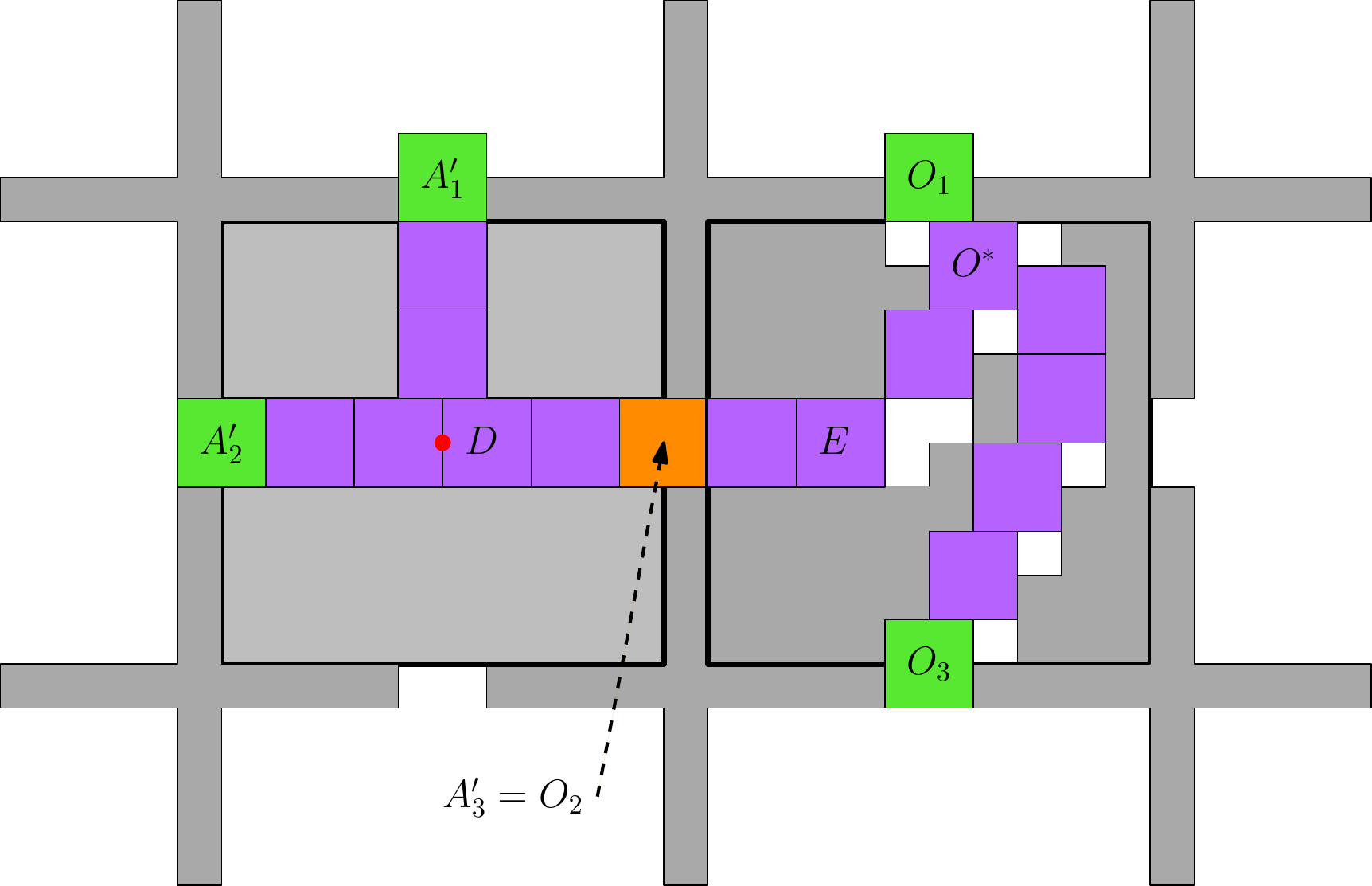}
    \caption{ }
    \label{fig:edge1}
  \end{subfigure} %
  \hspace{10pt}
  \begin{subfigure}[b]{0.45\textwidth}
    \includegraphics[width=\textwidth]{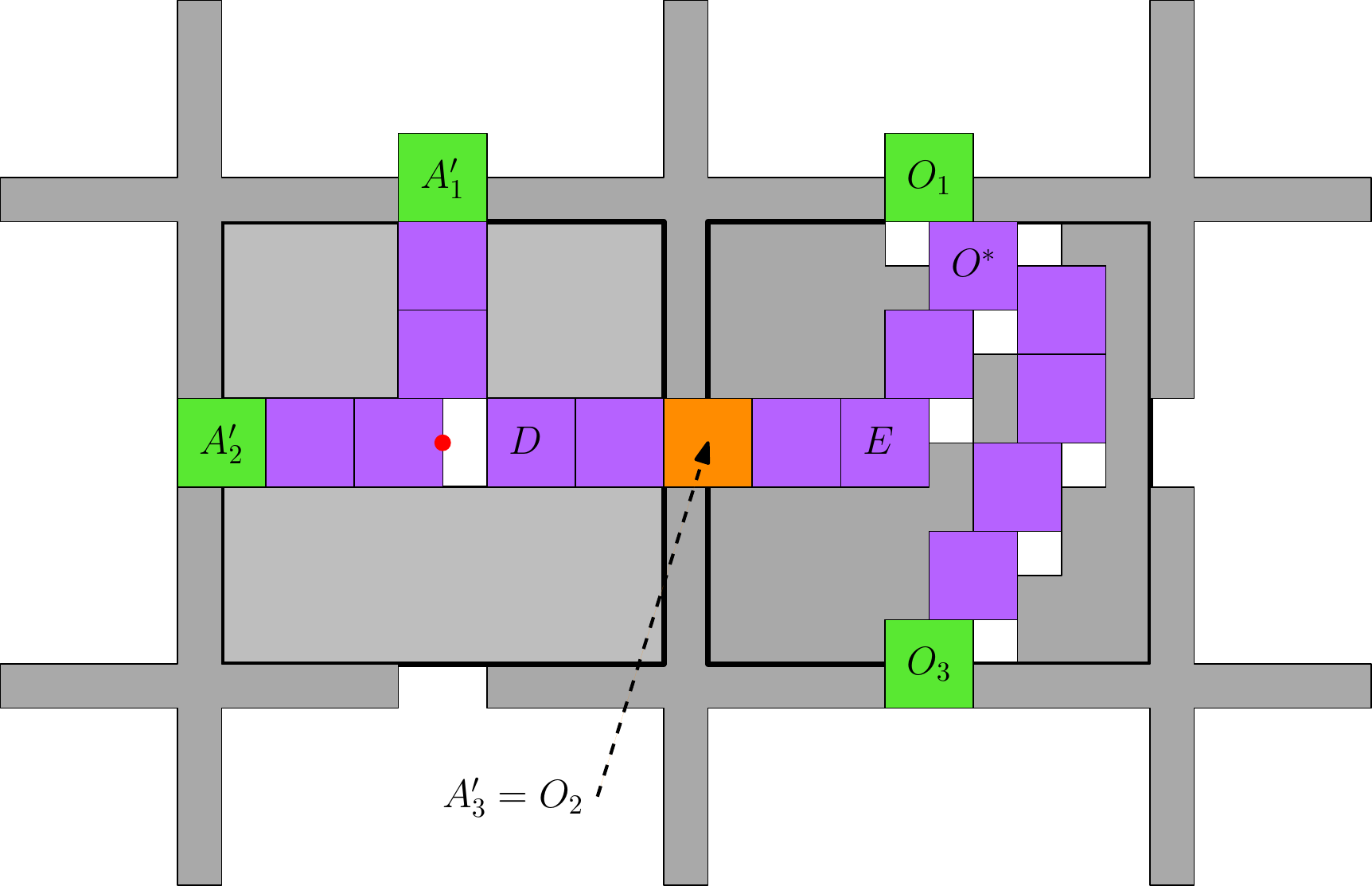}
    \caption{ }
    \label{fig:edge2}
  \end{subfigure}
  \caption{Illustration for Lemma~\ref{lem:restricted_edge} and a
    depiction of two gadgets that are connected through an edge robot. In both
    figures the same \AND gadget and \OR gadgets, which share an edge
    robot $A'_3=O_2$, are illustrated. These two gadgets simulate an
    \AND vertex and an \OR vertex which share an edge. For each edge
    robot $A'_i$ in the \AND gadget, denote by $a'_i$ the
    corresponding edge in the \AND vertex. Similarly, denote by $o_i$
    the edge corresponding to the edge robot $O_i$ in the vertex. Note
    that $a'_3$ and $o_2$ represent the same edge. In (a) the gadget
    represents a machine configuration in which $a'_1$ is directed
    inward, $a'_2$ outward, $a'_3$ outward with respect to the \AND
    vertex. As a consequence, $o_2$ is directed inward with respect to
    the \OR vertex. Additionally, $o_1$ is directed inward, while
    $o_3$ outward. In (b) a similar machine configuration is depicted,
    with the exception that $a'_3$ is now directed inward while $o_2$
    outward.}
  \label{fig:edge_restricted}
\end{figure*}

\begin{lemma}\label{lem:restricted_vertex}
  Each vertex robot can be in at most two distinct terminal
  configurations, save robot $O^*$ in the \OR gadget, which can be in
  at most three terminal configurations.
\end{lemma}

\begin{proof}
  First, note that every edge robot can move either horizontally or
  vertically. Let us consider for example robot $A'_3=O_2$ in
  Figure~\ref{fig:edge_restricted}, which can only move
  horizontally. Since this robot can only move between two terminal
  configurations (Lemma~\ref{lem:restricted_edge}), the two vertex
  robot that are directly to its right, cannot move further left than
  where they appear in Figure~\ref{fig:edge1}. Additionally, robot $E$
  is bounded from the right by an obstacle which does not allow it to
  move further to the right than it appears in
  Figure~\ref{fig:edge2}. A similar reasoning can be applied to all
  the other vertex robots. As to robot $O^*$, consider its position in
  Figure~\ref{fig:edge1}. It cannot go up as there is an obstacle
  there. If $O_3$ leaves the gadget it can move a half step right, and
  if $O_2$ leaves the gadget it can move a half step down. Notice
  however, that it cannot move simultaneously down and right since it
  will collide with the robot that is immediately to its right.
\end{proof}

The following lemma implies that our motion planning scenario is so
tight that each valid \mc can be interpreted in a single way as an
assignment of terminal configurations for the robots.
\begin{lemma}\label{lem:one_configuration} 
  Given a specific terminal configuration, there is at most one robot
  that can be in it.
\end{lemma}

We mention that a similar Lemma can be proven for the gadgets used by
Hearn and Demaine~\cite{hd-psb05} in their hardness proof of the
labeled variant of the \mtm problem. Since their proof uses two
different types of robots, their result can also be interpreted as a
\pspace-hardness proof for the $2$-color multi-robot motion planning
problem~\cite{sh-kcolor14}, which consists of two groups where the
robots within each group are identical and interchangeable.

We return to the \pspace-hardness proof of the four unlabeled problems
discussed in our paper.

\begin{thm}\label{thm:final}
  The problems \mtm, \mts, \sts, and \dtd, for unit-square robots
  translating amidst polygonal obstacles in the plane are all
  \pspace-hard.
\end{thm}
\begin{proof}
  We describe reductions from the three NCL-model decision problems to
  our four unlabeled problem. Specifically we show the following
  reductions:
  \begin{itemize}
  \item The \ctc problem, which is concerned with checking whether one
    NCL machine orientation can be transformed to another, is reduced
    to the \mtm unlabeled problem, which is concerned with deciding
    whether a collection of robots can be moved between two \mcs.
  \item The \cte problem, which is concerned with the existence of an
    orientation that can be reached from a given orientation where a
    direction of a specific edge is flipped, is reduced to the
    unlabeled \mts and \sts problems that are concerned with moving an
    arbitrary or a specific robot (respectively) to a given
    destination configuration, when a starting \mc is specified for
    all the robots.
  \item The \ete problem, which is concerned with the existence of two
    orientations, where one can be transformed to another, such that
    the two orientations have an opposite direction for a given edge,
    is reduced to the unlabeled \dtd problem which is concerned with
    the existence of two \mcs that are equivalent, where each of the
    two \mcs contains a specific configuration that is given as input
    (where the configuration is not assigned to a specific robot).
\end{itemize}
We use the same reduction for all three cases. Only the analysis
slightly differs. Given a grid-embedded constraint graph~$H$ we
generate a scenario for the unlabeled problem as we described above,
by placing gadgets corresponding to vertices and connecting gadget
according to the connections in~$H$.

We first note that given an orientation for $H$ it can be transformed
into a valid \mc consisting only of terminal configurations, where the
directions of edges in $H$ induce configurations for the respective
edge robots, and these in turn induce configurations for the vertex
robots. In the other direction, given a valid \mc for the robots that
consists of only terminal configurations a valid orientation for $H$
can be easily generated, by considering only the positions of edge
robots. Note that by Lemma~\ref{lem:restricted_edge} every edge robot
can be in at most two terminal configurations that represent two
opposite directions of the same edge of $H$.

We first prove the hardness of the \mtm problem by a reduction from
the $\ctc$ problem. Let $o_S,o_T$ be two orientations of $H$, and
denote by $S,T$ the two free \mcs induced by them. It is clear that if
$o_S\equiv o_T$ then also $S\equiv T$. Now, suppose that $S\equiv T$.
Notice that in order to show that this implies that $o_S\equiv o_T$ we
need to prove the existence of a solution $\Pi(S,T)$ where no two edge
robots move at a given time. More accurately, we need to show that
there exists a solution in which every edge robot is in transit
between two terminal configurations, only when all the other edge
robots are stationary in terminal configurations. We consider for
example the \AND gadget (Figure~\ref{fig:and1}) and show that each
simultaneous movement of edge robots, where several edge robots are
located simultaneously at non-terminal configuration, can be carried
out in a sequential manner as well. We treat the various combinations
of robots moving simultaneously in and out of the gadget.  If both
$A_2$ and $A_3$ move inside then $A_1$ must be out, and therefore the
former two robots do not depend on each other in order to make their
move. Now suppose that $A_2$ and $A_3$ simultaneously move
outside. This means that each of the two gadgets, to which $A_2$ and
$A_3$ are entering, already moved other edge robots, and vertex
robots, so that the entrance of $A_2,A_3$ will be possible. Thus the
fact that $A_2$ can leave the gadget does not depend on the fact that
$A_3$ leaves the gadget, and vice versa. Therefore, we can move
$A_2,A_3$ in a sequential manner.  Therefore, a solution $\Pi(S,T)$ as
required always exists and in the case that $S\equiv T$ it also
follows that $o_S\equiv o_T$.
	
We now proceed to prove the hardness of the \mts problem by a
reduction from the \cte problem. Recall that \cte consists of an
orientation $o_S$ and edge $e\in E_H$. The question is whether there
exists another orientation $o_T$ such that $o_S\equiv o_T$ and
$o_S(e)\neq o_T(e)$, i.e., the direction of $e$ in the two
configurations is reversed. Denote by $S$ the \mc induced by $o_S$,
and by $s\in S$ the terminal configuration that corresponds to the
edge $e$ in the direction $o_S(e)$. We now ask whether there exists a
\mc $T$ such that $S\equiv T$ such that there exists $t\in T$ that
corresponds to $e$ in the opposite direction.  Now that we have
defined the components of our \mts problem it is clear that if there
exists an orientation $o_T$ as required, then its induced \mc
satisfies the conditions of \mts. The opposite direction follows
similarly to the \mtm proof above.
	
The difference between \mts and \sts is that in the latter we require
that a specific robot will move to a specific target configuration.
Note that our reduction for \mts holds here as well, since a selection
of a specific edge of $H$ induces the selection of a specific robot
that has to move between two configurations (see
Lemma~\ref{lem:one_configuration}). The hardness of the \dtd problem
by a reduction from the \ete problem can be proved in a manner similar
to the previous three cases.
\end{proof}

We mention that using Lemma~\ref{lem:one_configuration}, our hardness
proof can also be applied to prove the hardness of the labeled variant
of \mtm. In this case, each robot $r_i$ is assigned with specific
start and target configurations $s_i$ and $t_i$, respectively, and the
goal is to move each $r_i$ from $s_i$ to $t_i$ while avoiding
collisions with robots and obstacles. Specifically, we have the
following result.

\begin{thm}
  Labeled multi-robot motion planning for unit-square robots moving
  amidst polygonal obstacles is \pspace-hard.
\end{thm}
\begin{proof}
  We follow the same reasoning as for the \ctc to \mtm reduction
  (Theorem~\ref{thm:final}) and add that due to
  Lemma~\ref{lem:one_configuration} given a placement of the robots
  induced by the starting \mc $S$, the targets \mc can be viewed as an
  assignment of a specific target for every robot. Specifically, given
  an initial assignment $s_i$ for robot $r_i$, there at most one
  configuration in $T$ to which it can move (see
  Lemma~\ref{lem:one_configuration}).
\end{proof}

\section{Concluding remarks}
In this paper we studied the problem of motion planning of  multiple unlabeled
unit-square robots in an environment cluttered with polygonal
obstacles. We proved that four variants of this problem are \pspace-hard.
While our result in itself is negative,
we hope that it will motivate research of other variants of the
unlabeled problem which may turn out to be polynomially solvable.

\bibliographystyle{plainnat}
\bibliography{bibilography.bib} 

\end{document}